\documentclass{article}

\usepackage{geometry}

\newgeometry{vmargin={15mm}, hmargin={35mm,35mm}}   

\usepackage{microtype}
\usepackage{graphicx}
\usepackage{subfigure}

\usepackage{url}            
\usepackage{booktabs}       

\usepackage[switch]{lineno}

\usepackage{amssymb,amsmath,amsthm,bm}

\newtheorem{theorem}{Theorem}

\newtheorem{lemma}{Lemma}

\usepackage{comment}
\usepackage{enumitem}

\def\vb{{\bm{b}}}

\def\ve{{\bm{e}}}

\def\vp{{\bm{p}}}

\def\vr{{\bm{p}}}

\def\vw{{\bm{w}}}
\def\vx{{\bm{x}}}
\def\vy{{\bm{y}}}
\def\vz{{\bm{z}}}

\def \Ib {{\mathbf{Ib}}}

\def \RR {{\mathbb{R}}}

\def \Ib {{\mathbf{I}}}
\def \Lb {{\mathbf{L}}}
\def \Ab {{\mathbf{A}}}

\def \Db {{\mathbf{D}}}
\def \Hb {{\mathbf{H}}}
\def \Pb {{\mathbf{P}}}
\def \Rb {{\mathbf{R}}}
\def \Tb {{\mathbf{T}}}
\def \Ub {{\mathbf{U}}}
\def \Zb {{\mathbf{Z}}}

\def \R {{\mathbb{R}}}

\def\p{\bm{w}}

\newcommand{\LL}{f}
\def\rr{\bm{r}}
\def\pp{\bm{p}}

\usepackage{verbatim}

\title{Stochastic Gradient Descent with Nonlinear Conjugate Gradient-Style Adaptive Momentum}

\author{
  Bao Wang \\
  Department of Mathematics\\
  Scientific Computing and Imaging Institute\\
  University of Utah\\
  \texttt{wangbaonj@gmail.com}\\
   \and
  Qiang Ye \\
  Department of Mathematics\\
  University of Kentucky\\
   \texttt{qye3@uky.edu}
}

\begin{document}

\maketitle

\begin{abstract}

Momentum plays a crucial role in stochastic gradient-based optimization algorithms for accelerating or improving training deep neural networks (DNNs). In deep learning practice, the momentum is usually weighted by a well-calibrated constant. However, tuning hyperparameters for momentum can be a significant computational burden. In this paper, we propose a novel \emph{adaptive momentum} for improving DNNs training; this adaptive momentum, with no momentum related hyperparameter required, is motivated by the nonlinear conjugate gradient (NCG) method. Stochastic gradient descent (SGD) with this new adaptive momentum eliminates the need for the momentum hyperparameter calibration, allows a significantly larger learning rate, accelerates DNN training, and improves final accuracy and robustness of the trained DNNs. For instance, SGD with this adaptive momentum reduces classification errors for training ResNet110 for CIFAR10 and CIFAR100  from $5.25\%$ to $4.64\%$ and $23.75\%$ to $20.03\%$, respectively. Furthermore, SGD with the new adaptive momentum also benefits adversarial training and 
improves adversarial robustness of the trained DNNs.

\end{abstract}

\section{Introduction}
Given a training dataset $\Omega_N:=\{\vx_i, y_i\}_{i=1}^N$ with $\vx_i$ and $y_i$ being the data-label pair of the $i$th instance. Natural training, i.e., training a machine learning classifier $y=g(\vx, \vw)$ for clean image classification, can be formulated as solving the following empirical risk minimization 
(ERM) problem \cite{vapnik1992principles}:
\begin{equation}
\label{eq:ERM}
{\min_{\vw\in \RR^d} f(\vw) := \frac{1}{N}\sum_{i=1}^N f_i(\vw) := \frac{1}{N}\sum_{i=1}^N \mathcal{L}(g(\vx_i, \vw), y_i),}
\end{equation}
where $\mathcal{L}$ is  typically the cross-entropy loss 
between the predicted label $g(\vx_i, \vw)$ and the ground truth $y_i$. 
To train a robust classifier under adversarial attacks \cite{goodfellow2014explaining,madry2017towards}, we often solve the following  
empirical adversarial risk minimization 
(EARM) problem:
\begin{equation}
\label{eq:EARM}
{\min_{\vw\in \RR^d} \frac{1}{N}\sum_{i=1}^N \max_{\|\vx_i-\vx_i'\|_2\leq \epsilon}\mathcal{L}(g(\vx_i', \vw), y_i),}
\end{equation}
with $\epsilon>0$ being a constant. Training deep neural networks (DNNs) by solving \eqref{eq:ERM} or \eqref{eq:EARM} is a very difficult task: 1) the objective function is highly nonconvex \cite{lecun2015deep}; 2) $N$ is very large, e.g., in ImageNet classification $N\sim 10^{6}$ \cite{russakovsky2015imagenet}, which makes computing the gradient of the loss function difficult and inefficient; 3) the dimension of $\vw$ is very high; for instance, in training ResNet200 for ImageNet classification, $\vw$ is of dimension $\sim 65M$  \cite{he2016deep}. Due to the above challenges, stochastic gradient descent (SGD) becomes the method of choice for training DNNs for image classification  \cite{bottou2018optimization}; momentum scaled by a well-calibrated weight is usually integrated with SGD to accelerate or improve training DNNs \cite{polyak1964some,sutskever2013importance,bengio2013advances,paszke2019pytorch,wang2020scheduled}. 



Starting from $\vw_0 \in \RR^d$ and $\vp_0 =0$,  in the $n (n\geq 1)$th iteration of SGD with momentum (scaled by a constant $\beta \geq 0$), we randomly 
sample a mini-batch $\{i_k\}_{k=1}^m \subset [N]\  (m\ll N)$;
update $\vw^{n}$ as follows \cite{paszke2019pytorch}: 
\begin{eqnarray}
\label{eq:ConstanMomentum}
{\begin{aligned}
  \vp_{n} &= \beta \vp_{n-1} + \frac{1}{m}\sum_{k=1}^m \nabla f_{i_k}(\vw_{n}),\\  
  \vw_{n+1} &= \vw_{n} - \alpha \vp_{n}, 
\end{aligned}}
\end{eqnarray}
where $\alpha > 0$  is the step size. $\nabla f_{i_k}(\vw_{n})$ in \eqref{eq:ConstanMomentum} can be replaced with $\nabla f_{i_k}(\vw_{n} - \alpha \vp_{n-1})$ to get the Nesterov momentum \cite{sutskever2013importance}. 
In training DNNs, the tuning of $\alpha$ 
and $\beta$ is time consuming. 
As a result, several adaptive learning rate 
algorithms have been developed and are widely used: Adagrad \cite{duchi2011adaptive} adapts learning rate to the parameters based on the sum of the squares of the gradients; Adadelta \cite{zeiler2012adadelta} and RMSprop \cite{hinton2012neural} modify Adagrad by restricting the window of accumulated past gradients to some fixed size; 
Adam integrates momentum with adaptive step size and achieves remarkable performances in many applications \cite{kingma2014adam,dozat2016incorporating,loshchilov2018fixing,reddi2019convergence}. However, there have not been many works on the varying or adaptive momentum ($\beta$). 
Nonlinear conjugate gradient (NCG) method \cite{fletcher1964function} 
can be considered as an 
adaptive momentum method combined with steepest descent along the search direction.
It significantly accelerates convergence of 
the gradient descent method and it has  some nice theoretical convergence guarantees
\cite{10.1093/imanum/5.1.121,doi:10.1137/0802003,dai1999nonlinear,hager2005new,powell1976some,Zoutendijk}. A major obstacle to using NCG is the need for a line search at each iteration; even an inexact search requires several function/derivative evaluations. As a result, it is not used very often in DNN applications \cite{Le_onoptimization}. However, the success of adaptive momentum in NCG and the success of SGD for DNN training motivate us to use the adaptive momentum from NCG to improve SGD based DNN training. 





\subsection{Our Contributions}
In this paper, we leverage the celebrated 
NCG-style adaptive momentum to accelerate GD/SGD; in particular, to improve training DNNs for image classification. We propose a GD/SGD method with a fixed learning rate but leverage an adaptive momentum as used in NCG. We will present some convergence results to show global convergence under certain conditions on the learning rate. For the case of quadratic functions, we will show the accelerated convergence rate under some quite general conditions.  We summarize the major advantages of SGD with adaptive momentum below:
\begin{itemize}[leftmargin=*]
\item It converges faster and allows us to use significantly larger step sizes to train DNNs.

\item It improves the accuracy and adversarial 
robustness of the trained DNNs for image classification.
For instance, it reduces test errors of training ResNet110 for CIFAR10 and CIFAR100 classification from $5.25\%$ to $4.64\%$ and $23.75\%$ to $20.03\%$, respectively. 
Furthermore, SGD with the new adaptive momentum also benefits adversarial training and hence improves adversarial robustness of the trained DNNs.

\item It eliminates the work for momentum-related hyperparameter 
tuning with almost no computational overhead.
\end{itemize}



\subsection{Related Works}
Non-constant momentum has been used to accelerate GD. One of the most exciting results is the Nesterov accelerated gradient (NAG) \cite{nesterov1983method,nesterov1998introductory}, which replaces the constant momentum with an iteration dependent 
momentum and achieves a convergence rate of  $O(1/k^2)$ for convex optimization (vs. GD with a convergence rate  $O(1/k)$). 
However, directly applying NAG to  
SGD suffers from error accumulation \cite{devolder2014first,wang2020scheduled}, which can be alleviated by using NAG with scheduled restart \cite{wang2020scheduled} at the cost of hyperparameter calibration. 


NCG \cite{fletcher1964function} is a popular optimization method that has been studied extensively. For various formulations, it has been proved for a general function to have a descent property and global convergence under some assumptions on the step size known as Wolf conditions; see  \cite{10.1093/imanum/5.1.121,dai1999nonlinear,doi:10.1137/0802003,hager2005new,powell1976some,Zoutendijk}. NCG has been applied to deep learning; 
\cite{Le_onoptimization} empirically compares NCG, L-BFGS, and the momentum methods and found each to be superior in some problems. There are some related works in avoiding the line search in NCG.
For example, \cite{Moller} uses some estimate of the Hessian to approximate optimal step size $\alpha$, while \cite{adya2018nonlinear} replaces the line search by starting with some initial $\alpha$ and then increasing or decreasing its value at each iteration by checking whether the loss  decreases or increases. 

\subsection{Organization}
We organize this paper as follows: In Section~ {\em  Algorithm: GD/SGD with Adaptive Momentum}, 
we give a brief review of  NCG, and then we leverage the adaptive momentum in NCG to accelerate GD and SGD. 
In Section ~\ref{sec:Theory}, 
we give the convergence guarantees of the proposed algorithms. We then present experiments to demonstrate the performance of the SGD with adaptive momentum for training DNNs in Section~\ref{sec:Experiments}. 
We end with some concluding remarks in the last section. 
Technical proofs and more experimental details 
are provided in the appendix.

\subsection{Notations}
We denote scalars by lower or upper case letters; vectors/ matrices by lower/upper case bold face letters. For a vector $\vx = (x_1, \cdots, x_d)^T\in \mathbb{R}^d$, we use $\|\vx\| = {(\sum_{i=1}^d |x_i|^2)^{1/2}}$ to denote its $\ell_2$ norm, and the $\ell_\infty$ norm of $\vx$ by $\|\vx\|_\infty = \max_{i=1}^d|x_i|$.
For a matrix $\Ab$, we use $\|\Ab\|_{2/\infty}$ to denote its induced norm by the vector $\ell_{2/\infty}$ norm. 
We denote the set $\{1, 2, \cdots, N\}$ as $[N]$. For a function $f(\vx): \mathbb{R}^d \rightarrow \mathbb{R}$, we denote $\nabla f(\vx)$ and $\nabla^2 f(\vx)$ the gradient and Hessian of $f(\vx)$, respectively.

\section{Algorithm: GD/SGD with Adaptive Momentum}\label{sec:Algorithm}
\subsection{Nonlinear Conjugate Gradient Methods}\label{subsec:NCG}
The classical GD with an optimal learning rate (step size) has a local convergence rate that depends on the condition number $\kappa$ of the Hessian matrix at a local minimum. 
The conjugate gradient (CG) method augments the gradient with a suitable momentum term as the search direction. In the quadratic case  ($\min_\vw f(\vw \in \RR^d) := 1/2\vw^T\Ab\vw + \vb^T\vw$ with $\Ab \in \RR^{d\times d}$ and $\vb \in \RR^d$ being a known matrix and vector, respectively.), the modified directions maintain orthogonality in the $\Ab$-inner product ($\vw^T\Ab\vw$), and it allows a significantly accelerated convergence rate that depends on $\sqrt{\kappa}$. It has been generalized to general nonlinear functions as follows. 


\begin{itemize}[leftmargin=*]
\item In the first iteration, perform a line search along $\vr_0:=\nabla f(\vw_0)$ to get the initial step size, i.e., $\alpha_0 := \arg\min_\alpha f(\vw_0 - \alpha\vr_0)$, and update $\vw$ by $\vw_1=\vw_0 -\alpha_0\vr_0$.
\item For $n\geq 1$, we perform the following updates for the $n$th iteration: 
\begin{itemize}[leftmargin=*]
    \item Compute $$\beta_n = \beta_n^{FR} := \frac{(\nabla f(\vw_n)^T \nabla f(\vw_n))}{(\nabla f(\vw_{n-1})^T \nabla f(\vw_{n-1}))}.$$





    \item Update the search direction: $$\vr_n = \nabla f(\vw_n) + \beta_n \vr_{n-1}.$$

    \item Perform a line search: $$\alpha_n = \arg\min_\alpha f(\vw_n - \alpha \vr_n).$$

    \item Update the position: $$\vw_{n+1} = \vw_n - \alpha_n \vr_n.$$
\end{itemize}
\end{itemize}

There are several possible formulations on $\beta_n$ in the literature, and the one we present $\beta_n^{FR}$ is known as the Fletcher-Reeves formula \cite{fletcher1964function}; see \cite{hestenes1952methods,dai1999nonlinear,polak1969note,shewchuk1994introduction,hager2005new} for other formulations and related theoretical properties. 
The NCG has been empirically found to have some similar convergence properties as the classical linear CG method. There have been several analyses to show a descent property and convergence of NCG for a general function under some forms of the Wolf conditions for inexact line search of $\alpha_n$; see  \cite{10.1093/imanum/5.1.121,doi:10.1137/0802003,dai1999nonlinear,hager2005new,powell1976some,Zoutendijk}. However, there appears to be no result characterizing its CG-like  accelerated convergence rate.

In the NCG method, a line search is performed to determine $\alpha_n$, while $\beta_n$ can be regarded as the momentum coefficient. Here, the momentum leverages the past gradient instead of the past position, which is different from the momentum in \eqref{eq:ConstanMomentum}. 
This has an advantage over the traditional momentum method in that the momentum coefficient is adaptively determined and no tuning is needed. A disadvantage of NCG is that even an inexact line search for $\alpha_n$ requires several function/gradient evaluations and would make the method less appealing for training DNNs. Therefore, NCG is rarely used for DNNS.

\subsection{(Stochastic) Gradient Descent with NCG Momentum}
The discussions in the previous parts motivate us to integrate the adaptive momentum coefficient $\beta_n$ into GD/SGD. We propose to consider GD/SGD with an adaptive momentum, i.e.,  with a fixed $\alpha$ but with $\beta=\beta_n^{FR}$ as the momentum coefficient at each step. This has two potential benefits. As a generalization of NCG, this may preserve some convergence properties of NCG. As a momentum method, there is no need to determine or tune the hyperparameter for momentum. We have chosen the Fletcher-Reeves formula $\beta_n^{FR}$ for its simplicity and robustness, as indicated by our preliminary numerical testing. We call the resulting algorithms FRGD/FRSGD, which we state as follows. Starting with $\vw_0$, we set $\vr_{-1}=\mathbf{0}$ and $\beta_{0} =0$ and iterate for $n\ge 0$ as follows: 

\begin{align}
\label{eq:FRGD}
\vr_{n} = \rr_n+ \beta_n \vr_{n-1},\\ \nonumber 
\vw_{n+1} = \vw_{n} - \alpha \vr_{n},
\end{align}
where 
\begin{itemize}
\item FRGD:
{$$ \rr_n = \nabla f(\vw_{n})$$ and 
 $$ \beta_{n} = (\rr_n^T\rr_n)/(\rr_{n-1}^T\rr_{n-1});$$}
\item FRSGD:
{$$ \rr_n = \frac{1}{m} {\sum_{j=1}^m} \nabla f_{i_j}(\vw_{n})$$ and 
 $$ \beta_{n} = (\rr_n^T\rr_n)/(\rr_{n-1}^T\rr_{n-1})$$.}
\end{itemize}





We have found this adaptive momentum method significantly accelerates the convergence of GD with momentum and outperforms NAG as well. The only extra computational cost over the momentum method is in computing an inner product $\rr_n^T \rr_n$ at each step, which is negligible. 

Before testing its performance in training DNNs, we first present an academic example to illustrate its potential advantage. 

\noindent {\sc Example 2.1:} We consider the following quadratic optimization problem \cite{mrtz}:
\begin{equation}\label{eq:quadratic}
{\small \min_\vw f(\vw) = \frac{1}{2}\vw^T\mathbf{L}\vw - \vw^T\vb,}
\end{equation}
where $\mathbf{L} \in \mathbb{R}^{500\times 500}$ is the Laplacian of a cycle graph, 
and $\vb$ is a $500$-dimensional vector whose first entry is $1$ and all the other entries are $0$.  It is easy to see that $f(\vw)$ is convex (not strongly convex) with Lipschitz constant $4$. We run GD, GD with momentum scaled by $0.9$ (GD $+$ Momentum), NAG, and FRGD with step size $1/4$ (the same hyperparameters as that used in \cite{mrtz}). 
As shown in Fig.~\ref{fig:quadratic}, GD $+$ Momentum converges faster than GD, while NAG speeds up GD $+$ Momentum dramatically and converges to the minimum in an oscillatory fashion. More interestingly, FRGD converges exponentially fast and significantly outperforms all other methods in this case. 


\begin{figure}[!ht]
\centering
\begin{tabular}{c}
\includegraphics[width=0.6\textwidth]{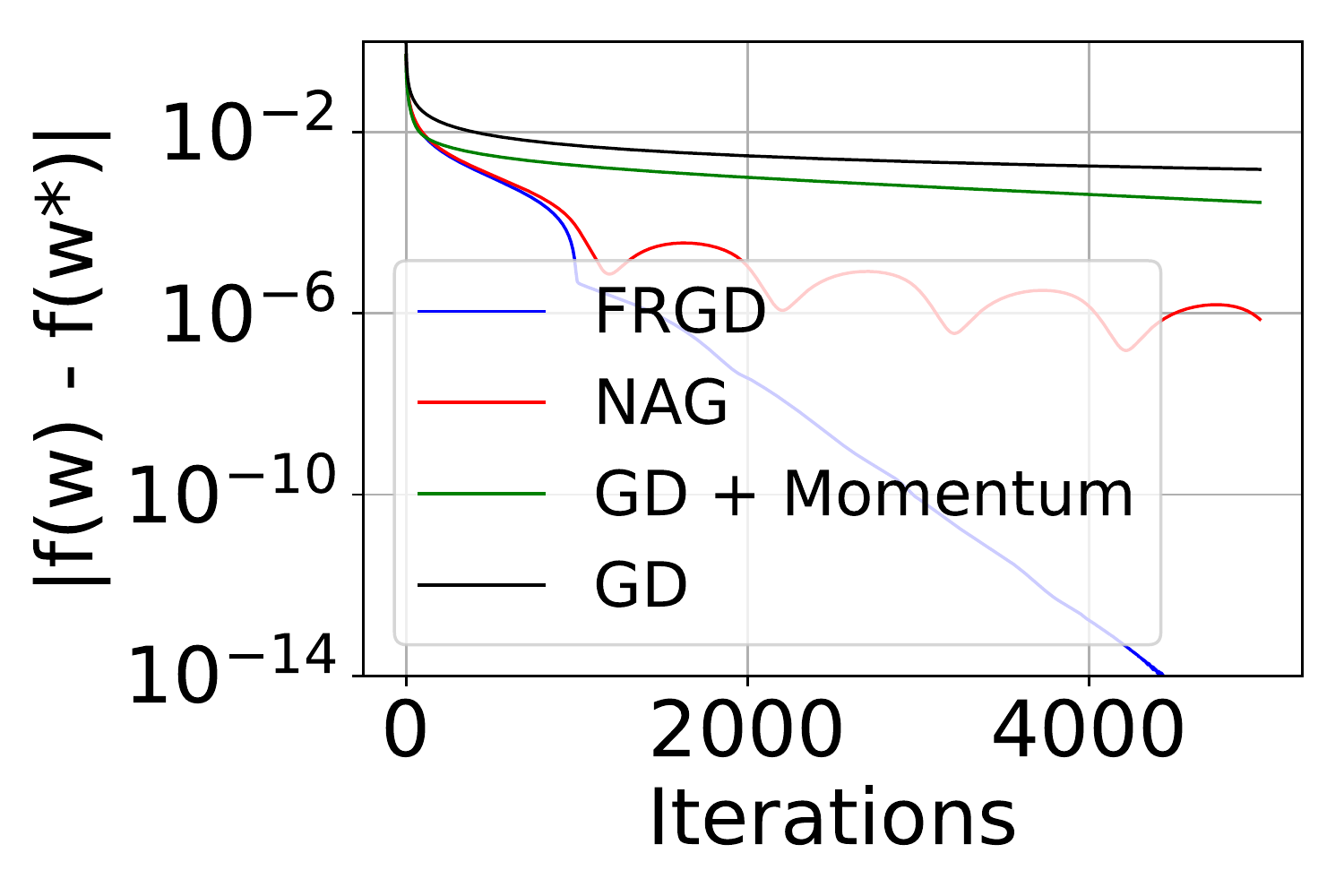}\\
\end{tabular}
\caption{Comparison between a few optimization algorithms for optimizing the quadratic function $f(\vw) = 1/2\vw^T\Lb\vw + \vb^T \vw,\  \vw, \vb\in \RR^{500}, \ \mbox{and}\ \Lb\in \RR^{500\times 500}$, where $\Lb$ is the Laplacian of a cycle graph, and $\vb$ is vector with the first entry being $1$ and all the others are $0$s \cite{mrtz}. 
  Momentum accelerates GD slightly; NAG oscillates to the minimum, $\vw^*$, and converges much faster than GD (with momentum); FRSGD converge almost exponentially fast to $\vw^*$.
}
\label{fig:quadratic}
\end{figure}

This example demonstrates that FRGD converges at a rate much faster than GD. We will present in the next section some theoretical results to demonstrate this property.



\section{Main Theory}\label{sec:Theory}
In this section, we present some convergence results to demonstrate descent property and   convergence of FRGD \eqref{eq:FRGD} under some conditions on $\alpha$. We shall focus on  strongly convex functions and quadratic functions. Our results are applicable only locally for a general function. For the case of quadratic functions, we present convergence bounds to demonstrate accelerated convergence rate. The proofs of these results will be given in Appendix~\ref{sec:appendix:proofs}. We first consider a convex function $f(\vw)$ with Lipschitz continuous Hessian matrix.

At each FRGD iteration, there exists $  \vy_i=(1-\xi_i)\p_{n} + \xi_i \p_{n+1}$ for some $\xi_i \in [0,1]$ ($1\le i \le d$) s.t.
\begin{equation}\label{eq:Hn}
{\nabla \LL (\p_{n+1})- \nabla \LL (\p_n) = \Hb_n (\p_{n+1}-\p_n),}
\end{equation}
where $\Hb_n=\left[\frac{\partial^2 \LL}{\partial w_i \partial w_j} (\vy_i) \right]_{i,j=1}^d$; see Lemma \ref{lem1} in the  Appendix~\ref{sec:appendix:proofs}.  This $\Hb_n$ is approximately a Hessian matrix.
\medskip
\begin{theorem}\label{thm1}
Consider the adaptive momentum method FRGD \eqref{eq:FRGD} for $\LL (\p): \R^d \rightarrow \R$. Assume that $\Hb_{ij}(\p):=\frac{\partial^2 \LL (\p)}{\partial w_i \partial w_j}$ is Lipschitz continuous with the Lipschiz constant $C$, i.e. $\|\Hb_{ij}(\p)- \Hb_{ij}(\tilde{\p})\| \le C \|\p-\tilde{\p}\|$ for all $1\le i, j \le d$. Let $\Hb_n$ be as defined in \eqref{eq:Hn}.
Assume for some $K>0$ that the eigenvalues of $\Hb_n$ are on $[\lambda_{\min}, \lambda_{\max}]$  with $\lambda_{\min}>0$ for $n \le K$  and $\alpha \le \frac{ \lambda_{\min}}{ (\lambda_{\max}^2+2 C d \|\rr_0\|) K^2}$. Then
\[
{\vr_n^T \rr_n > 0}
\]
and
\[
\|\rr_{n}\| \le \sqrt{1-\alpha \lambda_{\min}}  \|\rr_{n-1}\|,
\]
where $\rr_n = \nabla \LL (\p_n)$ and $n\le K$.
\end{theorem}
\medskip 

The theorem shows that $\vr_n$ is a descent direction and $\rr_n$ converges monotonically with a rate of at least $\sqrt{1-\alpha \lambda_{\min}}$.  Although such properties are expected for GD, it is important that with the adaptive momentum, FRGD maintains these properties. However, NAG does not converge monotonically to the minimum; instead it oscillates.
Moreover, it is worth noting that: 1) our assumption on the Hessian is different from the convergence theory of GD, we need Lipschitz Hessian while GD requires bounded Hessian; 2) both the step size constraint and the convergence rate, $\sqrt{1-\alpha \lambda_{min}}$, are also different from that of GD.


\subsection{
Quadratic Functions}
To further study the convergence rate, we consider a quadratic function $\LL (\vw)= \frac{1}{2}\vw^T \Ab\vw - \vb^T\vw$, where $\Ab$ is a symmetric positive definite matrix. In this case, the problem reduces to a linear system $\Ab\vw=\vb$. Then the classical CG method has been studied extensively and strong convergence results exist. With a fixed $\alpha$, most properties that the analysis of CG relies on no longer hold. Fortunately, a techniques used for the analysis of inexact CG  due to round off errors or inexact preconditioning~\cite{doi:10.1137/S1064827597323415,Tong2000AnalysisOT} can be adapted to our method. Using coupled two-term recurrences, our adaptive momentum becomes a Krylov subspace method, and we can derive the following convergence bound.
\medskip 
\begin{theorem}\label{thm2}
Consider  the adaptive momentum method FRGD \eqref{eq:FRGD} for $\LL (\vw)= \frac{1}{2}\vw^T \Ab\vw - \vb^T\vw$. Let $\Zb_n = [\vz_0, \vz_1, \cdots, \vz_{n-1}]$ where $\vz_i=\rr_i/\|\rr_i\|$. 
If  $\vz_0, \vz_1, \cdots, \vz_n$ 
are linearly independent,  then
\begin{equation}\label{eq:convergence:quadratic}
{ \|  \rr_{n}\| \le
2(1+K_n) \left( \frac{\sqrt{\kappa}-1}{\sqrt{\kappa}+1}\right)^n \| \rr_0\|,}
\end{equation}
where $K_n \le n(1+ n\rho/2) \|\Ab\| \kappa (\Zb_{n+1})$, $\rho=\max_{0\le j <i \le n-1} \|\rr_{i}\|^2/\|\rr_{j}\|^2$,  $\kappa$ is the spectral condition number of $\Ab$, and $\kappa (\Zb_{n+1})$ is the spectral condition number of $\Zb_{n+1}$.
\end{theorem}
\medskip
The bound in \eqref{eq:convergence:quadratic} contains a linearly converging term with the rate $(\sqrt{\kappa}-1)/(\sqrt{\kappa}+1)$, 
but it also depends on the term $K_n$, which may grow with $n$. The key term in $K_n$ is $\kappa(\Zb_{n+1}) =\| \Zb_{n+1}\| \| \Zb_{n+1}^+\| \approx \sqrt{n+1} \| \Zb_{n+1}^+\|$, which measures linear independence among $\vz_0, \vz_1, \cdots, \vz_n$. Thus, as long as $\vz_0, \vz_1, \cdots, \vz_n$ does not completely lose linear independence, $K_n$ may be a modestly increasing term so that $\|  \rr_{n}\| $ converges at a rate close to $(\sqrt{\kappa}-1)/(\sqrt{\kappa}+1)$. 
Note that $\rho$ may be expected to be bounded. In particular, if $\|\rr_n\|$ is monotonic, which holds under the condition of Theorem \ref{thm1}, then $\rho \le 1$. 

Note that the classical CG method converges at the rate of $(\sqrt{\kappa}-1)/(\sqrt{\kappa}+1)$ and so does the momentum method with the following optimal $\alpha$ and $\beta$:
\[
{\alpha =\frac{4}{(\sqrt{\lambda_{\max}}+\sqrt{\lambda_{\min}})^2}, }
\]
\[
\beta=\frac{(\sqrt{\lambda_{\max}}-\sqrt{\lambda_{\min}})^2}{(\sqrt{\lambda_{\max}}+\sqrt{\lambda_{\min}})^2},
\]
where $\lambda_{\max}$ and $\lambda_{\min}$ are the largest and the smallest eigenvalues of $A$ respectively; see \cite{Polyak1964SomeMO}. However, they require strong conditions with the former requiring variable $\alpha_n$ (see its definition in the previous section),
while the latter requiring the optimal $\alpha$ and $\beta$.  All these methods significantly accelerates GD with constant step size $\alpha$, which has at best a convergence rate of  $(\kappa-1)/(\kappa+1)$.

We remark that a slightly more general bound $ \|  \rr_{n}\| \le
(1+K_n) \min_{p\in {\cal P}_n, p(0)=1} \|p(\Ab) \rr_0\|$ holds in place of (\ref{eq:convergence:quadratic}) without the assumption that $\Ab$ is symmetric positive definite; see the proof in the Appendix.
This bound and thus our method are applicable to the more general situations of positive semi-definite $\Ab$  (Example 2.1) or the trust region problem with indefinite $\Ab$ \cite{NIPS2018_8269,doi:10.1137/16M1095056}.

We also note that the theorem also \emph{does not require any explicit condition on the learning rate $\alpha$}. Although $\alpha$ may affect the quality of the basis $\vz_0, \vz_1, \cdots, \vz_n$ generated, as long as $K_n$ increases gradually at a rate slower than  
$(\sqrt{\kappa}-1)/(\sqrt{\kappa}+1)$, 
we have convergence of $\rr_n$. This may explain the success of our method with quite large learning rates (see our numerical results in the Experiment section). 
Of course, this is only true to the extent that $\alpha$ is not so large that  the condition number of the basis generated grows unbounded.





\section{Experiments}\label{sec:Experiments}
In this section, we present numerical results to illustrate the advantage of FRSGD over the baseline SGD with constant momentum in training DNNs for image classification. We run all experiments with five independent random seeds and report the means and standard deviations of the results. 

\begin{figure*}[!th]
\centering
\begin{tabular}{cc}
\includegraphics[clip, trim=0cm 0cm 0cm 0cm, width=0.4\columnwidth]{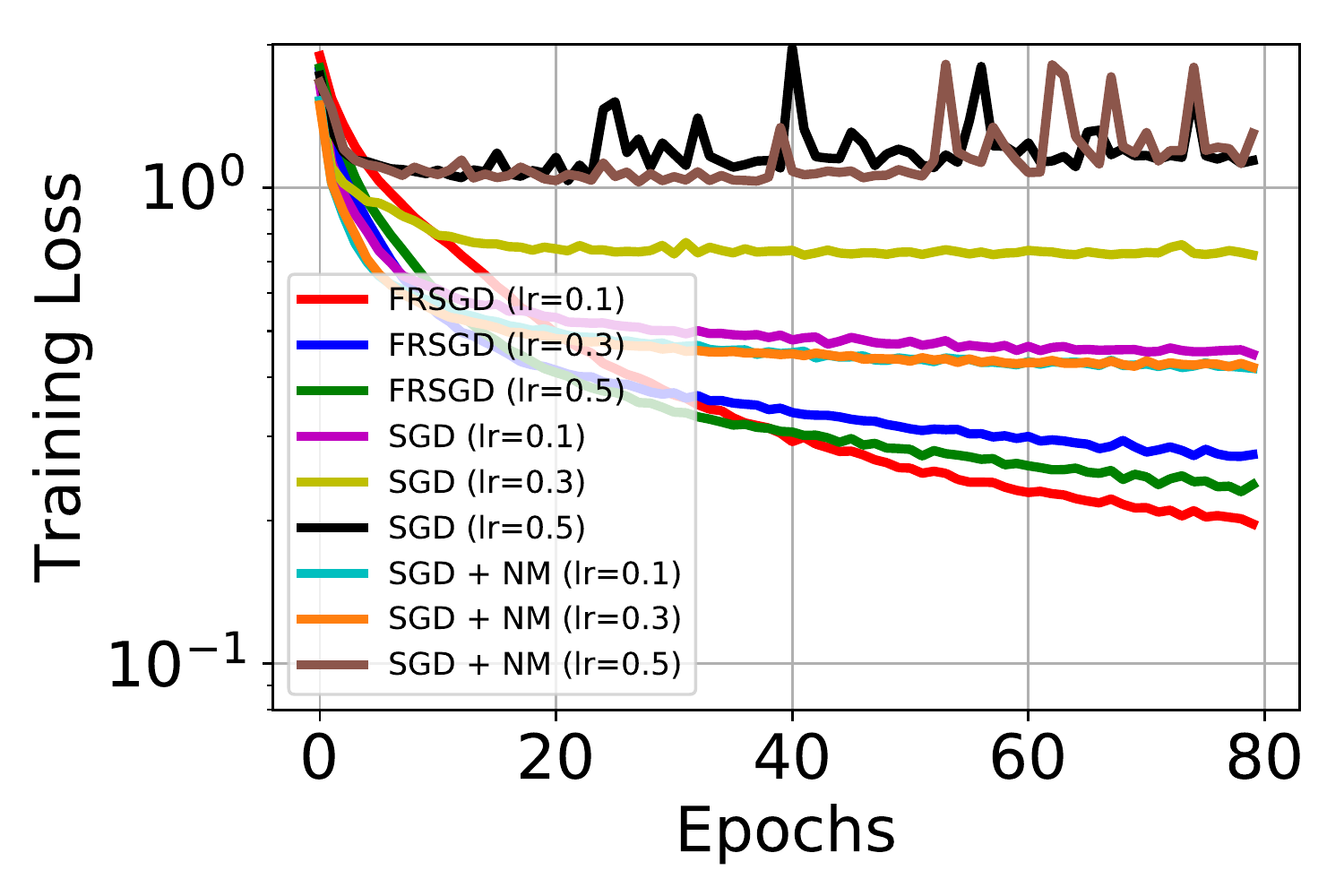}&
\includegraphics[clip, trim=0cm 0cm 0cm 0cm, width=0.4\columnwidth]{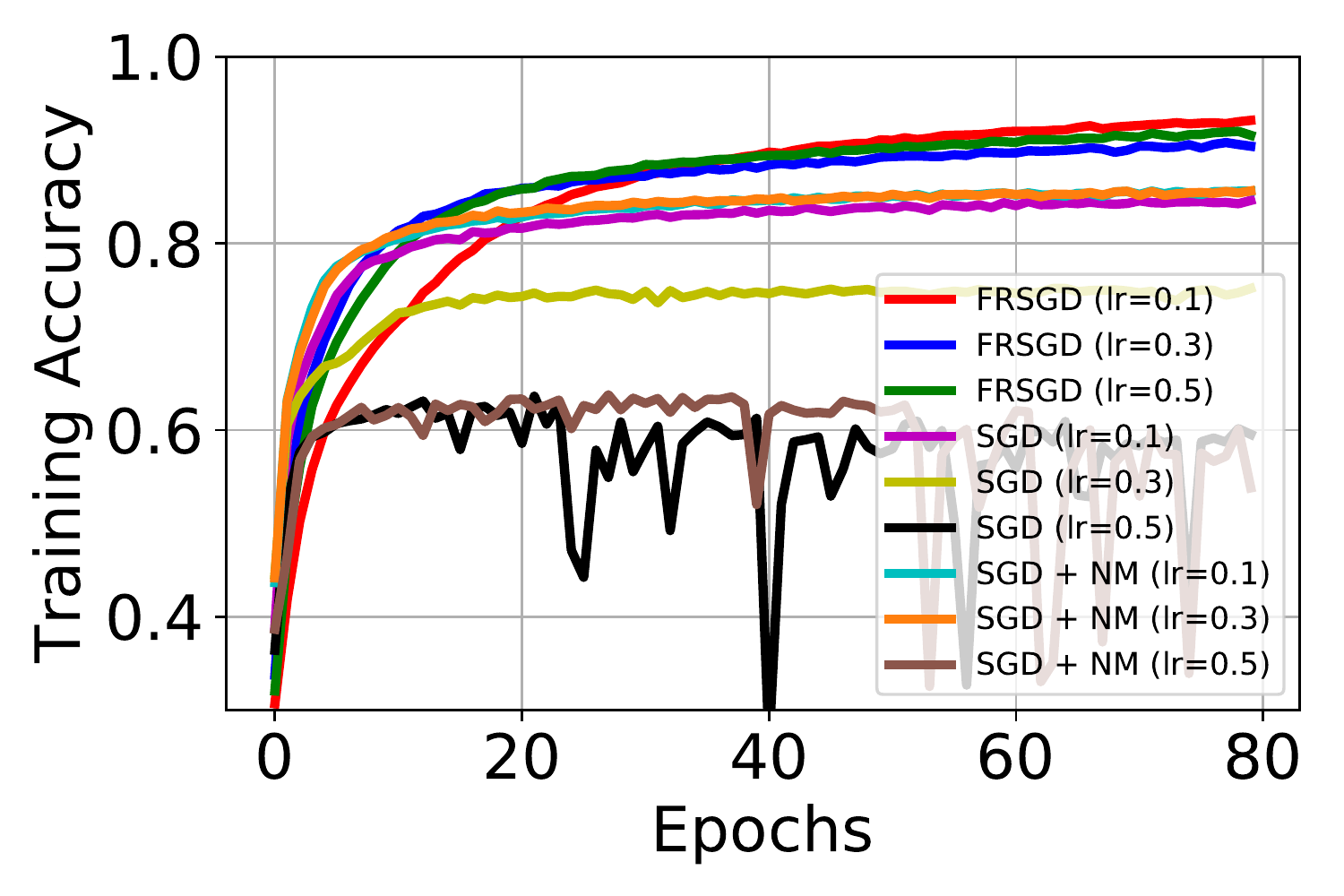}\\
\end{tabular}
\caption{Plots of epochs vs. training loss and accuracy of PreResNet56 trained by FRSGD, SGD (with momentum), and SGD $+$ NM with different learning rate. FRSGD is the most robust optimizer to large learning rate; where the performance of SGD and SGD $+$ NM deteriorate when a large learning rate is used. 
}
\label{fig:Loss:Curves:Different:LR}
\end{figure*}

\paragraph{Objective.} Our experimental results will demonstrate the following advantages of FRSGD over the baseline methods of SGD with momentum 
or Nesterov momentum: 1) FRSGD converges significantly faster; 2) FRSGD is significantly more robust to large step sizes; 3) DNN trained by FRSGD is more accurate and more adversarially robust than that trained by the baseline methods.
\paragraph{Datasets.} 
We consider {\small CIFAR10 \& CIFAR100} datasets \cite{krizhevsky2009learning} due to limited computing resources. 
Both datasets consist of 60K 32 by 32 color images with 50K/10K training/test split. CIFAR10/CIFAR100 contains 10/100 different classes, with each class having the same number of images.

\paragraph{Tasks, Experimental Settings, and Baselines.} We consider both natural and adversarial training, by solving \eqref{eq:ERM} and \eqref{eq:EARM}, respectively. We use pre-activated ResNets (PreResNets) models of different depths \cite{he2016identity}. 
As baseline methods, we consider SGD with the standard momentum and with the Nesterov momentum scaled by $0.9$, and we denote them as SGD and SGD $+$ NM, respectively, in the following context. We note that SGD is the optimizer used in the original ResNet implementations \cite{he2016deep,he2016identity}.
For the SGD and SGD $+$ NM baselines, we follow the standard-setting of ResNets by running it for $200$ epochs with an initial learning rate of $0.1$ and decay it by a factor of $10$ at the $80$th, $120$th, and $160$th epoch, respectively. For FRSGD, we run $240$ epochs with an initial learning rate of  $0.5$ and reduce it by a factor of $10$ at the $180$th, $220$th, and $230$th epoch, respectively \footnote{Here, we are able to use a larger learning rate under which FRSGD is still stable, and we decay the learning rate when the decrease of the loss function becomes slower. 
}. For adversarial training, we use the same SGD and FRSGD solvers described above to solve the outer minimization problem, and we run $10$ iterations of the iterative fast gradient sign method (IFGSM$^{10}$) attack with $\alpha=2/255$ and $\epsilon = 8/255$ to approximate the solution of the inner maximization problem. We provide the details of IFGSM \cite{Goodfellow:2014AdversarialTraining} and a few other attacks in the Appendix.




\subsection{FRSGD is Robust Under Large Learning Rates}\label{subsec:large:lr}
In this subsection, we  compare the performance of SGD, SGD $+$ NM, and FRSGD in training PreResNet56 for CIFAR10 classification using different learning rates.
We set the learning rate to be 0.1, 0.3, and 0.5, with all the other parameters the same as before. We plot epochs vs. training loss and training accuracy in Fig.~\ref{fig:Loss:Curves:Different:LR}. These results show that: 1) under the same small learning rate, e.g., 0.1, FRSGD converges remarkably faster than both SGD and SGD $+$ NM; 2) the convergence of SGD and SGD $+$ NM gets deteriorated severely when a larger learning rate is used; in particular, the training loss will not converge when 0.5 is used as the learning rate. However, FRSGD maintains convergence even when a very large learning rate is used; 3) the training loss and accuracy curves of SGD and SGD $+$ NM get plateau very quickly at a large loss, while FRSGD continues to decay.


\subsection{FRSGD Improves Accuracy of DNNs for Image Classification}
\paragraph{CIFAR10.} We consider training PreResNets with different depths using the settings mentioned before for the CIFAR10 classification. We list the test errors of different ResNets trained by different stochastic optimization algorithms in Table~\ref{tab:cifar10-acc}. In general, SGD performs on par with SGD $+$ NM; SGD $+$ NM has small advantages over SGD for training shallow DNNs. FRSGD outperforms SGD by $0.5 \sim 0.7$\% for ResNets with the depth ranging from 56 to 470. These improvements over already small error rates are significant in the relative sense, e.g., for PreResNet470, the relative error reduction is $\sim 13$\% ($4.92$\% vs. $4.27$\%).


\renewcommand\arraystretch{1.3}
\setlength{\tabcolsep}{5pt}
\begin{table*}[t!]
\caption{Test error (\%) on CIFAR10 using the SGD (with momentum), SGD $+$ NM, and FRSGD. We also include the reported results from \cite{he2016identity} (in parentheses) in addition to our reproduced results. ResNets trained by FRSGD is consistently more accurate than those trained by SGD and SGD $+$ NM.}\label{tab:cifar10-acc}
\vspace{1.0mm}
\centering
\fontsize{7.5pt}{0.95em}\selectfont
\begin{tabular}{c|c|c|c|c|c}
\hline
Network 
&SGD (baseline) 
& SGD$+$NM & FRSGD  & Improve over SGD & Improve over SGD$+$NM\\
\hline
PreResNet56 
& $6.12 \pm 0.24$ & $5.90 \pm 0.17$ & $\pmb{5.39 \pm 0.13}$ & $0.73$ &  $0.51$ \\
PreResNet110 
& $5.25 \pm 0.14$ ($6.37$) & $5.24 \pm 0.16$ & \pmb{$4.73 \pm 0.12$} & $0.52$ & $0.51$ \\
PreResNet164 
& $5.10 \pm 0.19$ ($5.46$) & $5.08 \pm 0.21$ &\pmb{$4.50 \pm 0.16$} & $0.60$ & $0.58$ \\
PreResNet290 
& $5.05 \pm 0.23$ & $5.04 \pm 0.12$ & \pmb{$4.44 \pm 0.10$} & $0.61$ & $0.60$ \\ 
PreResNet470 
& $4.92 \pm 0.10$ & $4.97 \pm 0.15$ & \pmb{$4.27 \pm 0.09$} & $0.65$ & $0.70$ \\
\hline
\end{tabular}
\end{table*}

\paragraph{CIFAR100.} Here, we consider CIFAR100 classification with the same DNNs and the same settings as those used for CIFAR10 classification. We report the test errors in Table~\ref{tab:cifar100-acc}. In this case, FRSGD improves the test accuracy over both SGD and SGD $+$ NM by $\sim 1.0$ to $1.6$\%.
\renewcommand\arraystretch{1.3}
\setlength{\tabcolsep}{5pt}
\begin{table*}[t!]
\caption{Test error (\%) on CIFAR100 using the SGD (with momentum), SGD $+$ NM, and FRSGD. We also include the reported results from \cite{he2016identity} (in parentheses) in addition to our reproduced results. ResNets trained by FRSGD is uniformly more accurate than those trained by SGD and SGD $+$ NM.}\label{tab:cifar100-acc}
\vspace{1.0mm}
\centering
\fontsize{7.5pt}{0.95em}\selectfont
\begin{tabular}{c|c|c|c|c|c}
\hline
Network 
&SGD (baseline) & SGD$+$NM & FRSGD  & Improve over SGD & Improve over SGD$+$NM\\
\hline
PreResNet56 
& $26.60 \pm 0.33$ & $26.14 \pm 0.38$ & $\pmb{25.00 \pm 0.32}$ & $1.60$ & $1.14$ \\
PreResNet110 
& $23.75 \pm 0.20$ & $23.65 \pm 0.36$ &\pmb{$22.52 \pm 0.35$} & $1.23$ & $1.13$ \\
PreResNet164 
& $22.76 \pm 0.37$  ($24.33$) & $22.79 \pm 0.29$ &\pmb{$21.38 \pm 0.34$} & $1.38$ & $1.41$ \\
PreResNet290 
& $21.78 \pm 0.21$ & $21.68 \pm 0.21$ & \pmb{$20.66 \pm 0.31$} & $1.12$ & $1.02$ \\
PreResNet470 
& $21.43 \pm 0.30$ & $21.21 \pm 0.30$ & \pmb{$19.92 \pm 0.29$} & $1.51$ & $1.29$ \\
\hline
\end{tabular}
\end{table*}


\renewcommand\arraystretch{1.3}
\setlength{\tabcolsep}{5pt}
\begin{table*}[t!]
\caption{Lists of the optimal training/test loss and accuracy of PreResNet110 trained by FRSGD, SGD (with momentum), SGD $+$ NM, and Adam with $240$ epochs. Adam has smaller training loss than the others, but the PreResNet110 trained by FRSGD has the smallest test loss/error.}\label{tab:compare:240epochs}
\vspace{1.0mm}
\centering
\fontsize{7.5pt}{0.95em}\selectfont
\begin{tabular}{c|c|c|c|c}
\hline
\ \ \ Optimizer\ \ \  & \ \ \  Training Loss\ \ \  &\ \ \  Training Error Rate (\%)\ \ \  &\ \ \  Test Loss\ \ \  &\ \ \  Test Error rate (\%)\ \ \  \\
\hline
SGD  & $0.00529 \pm 0.00043$ & $0.042 \pm 0.006$ & $0.1950 \pm 0.00091$ & $5.23 \pm 0.15$ \\
SGD $+$ NM  & $0.00462 \pm 0.00047$ & $0.032 \pm 0.005$ & $0.1846 \pm 0.00101$ & $5.19 \pm 0.16$ \\
Adam  & $\pmb{0.00033 \pm 0.00003}$ & $\pmb{0.002 \pm 0.002}$ & $0.3237 \pm 0.00125$ & $6.47 \pm 0.31$ \\
FRSGD  & $0.00680 \pm 0.00021$ & $0.090 \pm 0.002$ & $\pmb{0.1611 \pm 0.00086}$ & $\pmb{4.73 \pm 0.12}$ \\
\hline
\end{tabular}
\end{table*}

\paragraph{FRSGD vs. Adam and SGD with More Epochs.} 
In the original ResNet experiments, we have run SGD and SGD $+$ NM for 200 epochs, after which no training loss decay is observed. 
To compare over longer iterations, we train PreResNet110 by running SGD, SGD $+$ NM, and FRSGD for 240 epochs \footnote{Based on trial and error, we found that adding 20 epochs each in the first and second learning rate stages gives the best performance. All the reported results are based on using this setting.}. Moreover, we also compare them with running Adam \cite{kingma2014adam} for 240 epochs using a well-calibrated initial learning rate 0.003 and the same decaying schedule 
as SGD with 240 epochs, and we use the default value for all Adam's other hyperparameters. Table~\ref{tab:compare:240epochs} lists the training and testing losses as well as errors of PreResNet110 trained by different optimizers on CIFAR10. 
The best test error of SGD, SGD $+$ NM, and Adam for CIFAR10 classification 
are $5.23 \pm 0.15$\%, $5.19 \pm 0.16$\%, and $6.47 \pm 0.31$\% respectively, compared with $4.73 \pm 0.12$ for FRSGD.  We see that adding 40 more epochs to SGD and SGD $+$ NM does not improve classification accuracy much, and this is because the training loss has reached the plateau at each stage with a budget of 200 epochs. Adam converges faster with a smaller final training loss, but the testing loss and accuracy are far behind those of SGD, SGD $+$ NM, and FRSGD.

\renewcommand\arraystretch{1.3}
\setlength{\tabcolsep}{5pt}
\begin{table*}[t!]
\caption{ 
Test accuracy (\%) of PreResNet110 on CIFAR10 using PGD adversarial training with SGD $+$ NM and FRSGD as the outer solver. FRSGD improves accuracies for classifying both clean and adversarial images.}\label{tab:cifar10-acc-robust-resnet110}
\vspace{1.0mm}
\centering
\fontsize{7.5pt}{0.95em}\selectfont
\begin{tabular}{c|c|c|c|c|c|c|c}
\hline
Optimizer &  Natural & FGSM & IFGSM$^{10}$ & IFGSM$^{20}$ & IFGSM$^{40}$ & IFGSM$^{100}$ & C\&W \\
\hline
SGD $+$NM  & $82.19 \pm 0.29$ & $57.61 \pm 0.33$ & $55.35 \pm 0.42$ & $52.02 \pm 0.34$ & $51.45 \pm 0.33$ & $51.08 \pm 0.35$ & $62.92 \pm 0.50$\\
FRSGD & \pmb{$82.36 \pm 0.27$} & \pmb{$58.27 \pm 0.29$} & \pmb{$55.83 \pm 0.31$} & \pmb{$53.07 \pm 0.28$} & \pmb{$52.39 \pm 0.25$} & \pmb{$52.15 \pm 0.19$} & \pmb{$63.05 \pm 0.33$}\\
\hline
\end{tabular}
\end{table*}



\renewcommand\arraystretch{1.3}
\setlength{\tabcolsep}{5pt}
\begin{table*}[t!]
\caption{ 
Test accuracy (\%) of PreResNet110 on CIFAR100 using PGD adversarial training with SGD $+$ NM and FRSGD as the outer solver. FRSGD improves accuracies for classifying both clean and adversarial images.}\label{tab:cifar100-acc-robust-resnet110}
\vspace{1.0mm}
\centering
\fontsize{7.5pt}{0.95em}\selectfont
\begin{tabular}{c|c|c|c|c|c|c|c}
\hline
Optimizer &  Clean & FGSM & IFGSM$^{10}$ & IFGSM$^{20}$ & IFGSM$^{40}$ & IFGSM$^{100}$ & C\&W \\
\hline
SGD & $54.75 \pm 0.52$ & $30.75 \pm 0.41$ & $29.61 \pm 0.45$ & $27.87 \pm 0.44$ & $27.51 \pm 0.42$ & $27.40 \pm 0.48$ & $38.97 \pm 0.66$\\
FRSGD & \pmb{$54.95 \pm 0.49$} & \pmb{$31.77 \pm 0.43$} & \pmb{$30.79 \pm 0.33$} & \pmb{$29.32 \pm 0.37$} & \pmb{$29.09 \pm 0.40$} & \pmb{$29.01 \pm 0.39$} & \pmb{$39.01 \pm 0.50$}\\
\hline
\end{tabular}
\end{table*}
\medskip

\paragraph{Training FRSGD with a Large Number of Epochs.} 
In the previous experiments, considering the training efficiency, we limited the budget for training epochs of FRSGD by dropping the learning rate when the training loss convergence slows down but before reaching plateaus. This learning rate reduction may be premature and the result may not be the best accuracy our method can achieve.  
In this experiment, we relax this budget and use a much larger number of epochs for FRSGD and see if we can get more improvement in classification accuracy. In particular, we train PreResNet110/PreResNet290 by running 400 epochs of FRSGD with an initial learning rate of 0.5 and reduce the learning rate by a factor of 10 at the 200th, 300th, and 350th epochs, respectively. 

In this setting, we get the best test error rates of $4.64 \pm 0.12/4.26 \pm 0.09$\%  for CIFAR10 and $20.03 \pm 0.29/19.89 \pm 0.19$\% for CIFAR100, which  remarkably improves what we get by using 240 epochs ($4.73 \pm 0.12/4.44 \pm 0.10$ for CIFAR10 and $22.52 \pm 0.35/20.66 \pm 0.31$ for CIFAR100). Furthermore, the training loss of FRSGD at the last epoch in training the PreResNet290 for CIFAR10 classification also becomes significantly smaller than that of SGD or SGD $+$ NM ($0.00138 \pm 0.00012$ (FRSGD), vs. $0.0048 \pm 0.0003$ (SGD), and $0.0052 \pm 0.0003$ (SGD $+$ NM)).



\subsection{FRSGD Improves Adversarial Training}
Finally, we numerically demonstrate that FRSGD can also improve the adversarial robustness of the trained DNNs through adversarial training. We train the PreResNet110 by applying the adversarial training using the settings listed before. Then we apply the well-trained PreResNet110 to classify the test set under three kinds of benchmark adversarial attacks: fast gradient sign method (FGSM), $m$ steps IFGSM (IFGSM$^{m}$ with $m=10, 20, 40, 100$) \cite{Goodfellow:2014AdversarialTraining}, and C\&W attacks \cite{CWAttack:2016}. We apply the same set of hyperparameters for these attacks as that used in \cite{Wang:2018EnResNet,madry2017towards} in the following experiments. A brief introduction of these attacks and the used hyperparameters are available in Appendix. 

Tables~\ref{tab:cifar10-acc-robust-resnet110} and \ref{tab:cifar100-acc-robust-resnet110} list the accuracy of the adversarially trained PreResNet110 for classifying CIFAR10 and CIFAR100 images with or without adversarial attacks \footnote{We only compare FRSGD with SGD $+$ NM since SGD performance is weaker than SGD $+$ NM in this case.}. First, we see that the robust PreResNet110 trained by FRSGD is slightly more accurate than that trained by SGD $+$ NM for classifying the clean CIFAR10 and CIFAR100 images without any attack, e.g., the accuracy of FRSGD is $82.36 \pm 0.27$\% and $54.95 \pm 0.49$\% for CIFAR10 and CIFAR100 classification, while the corresponding accuracy of the model trained by SGD $+$ NM is $82.19 \pm 0.29$\% and $54.75 \pm 0.52$\%, respectively. Second, the model trained by FRSGD is more robust than that trained by SGD $+$ NM under all the adversarial attacks mentioned before, e.g., under the IFGSM$^{100}$ attack, the robust accuracy of these two models are $52.15 \pm 0.19$\% vs. $51.08 \pm 0.35$\% for CIFAR10 classification, and are $27.40 \pm 0.48$\% vs. $29.01 \pm 0.39$\% for CIFAR100 classification.
Although the improvements are minor, the good performance of FRSGD in this difficult setting illustrates its robustness in different problem types. 




\section{Concluding Remarks}\label{sec:Conclusion}
In this paper, we leveraged adaptive momentum from the NCG to improve SGD, and the resulting algorithm performs surprisingly well in the following sense: 1) It can accelerate GD significantly; in particular, we observed that it achieves exponential convergence for optimizing a specific convex function; 2) It allows us to use much larger step sizes and converges faster than SGD with (Nesterov) momentum in training DNNs; 3) DNNs trained by FRSGD have remarkably higher classification accuracy and are more robust to adversarial attacks for image classification. The method is as simple as SGD and is easy to implement. It is well suited for DNN training.

There are several interesting open problems that are worth further investigations. First, can we integrate the adaptive momentum with adaptive step size to further improve stochastic optimization algorithms? Second, can we prove stronger convergence results for FRGD/FRSGD under more general conditions?  Third, can we leverage adaptive momentum to improve training DNNs for other deep learning tasks beyond image classification? These will potentially be our future works.




\clearpage

\appendix
\section{Proof of Main Theorems}\label{sec:appendix:proofs}

We first give a lemma and then prove Theorem \ref{thm1}. For convenience, we restate all the theorems. 
\begin{lemma}\label{lem1}
Assume  $\frac{\partial^2 \LL}{\partial w_i \partial w_j}$ is Lipschitz continuous with the Lipschiz constant $C$ for all $1\le i, j \le d$, i.e. $|\frac{\partial^2 \LL}{\partial w_i \partial w_j} (\p)- \frac{\partial^2 \LL}{\partial w_i \partial w_j}(\tilde{\p})| \le C \|\p -\tilde{\p}\| $. For each fixed $n\ge 0$, there exists  $\vy_i=(1-\xi_i)\p_{n} + \xi_i \p_{n+1}$ for some $\xi_i \in [0,1]$ ($1\le i \le d$) s.t.
\[
\nabla \LL (\p_{n+1})- \nabla \LL (\p_n) = \Hb_n (\p_{n+1}-\p_n) \;\;\mbox{ where }\;\;
H_n=\left[\frac{\partial^2 \LL}{\partial w_i \partial w_j} (\vy_i) \right]_{i,j=1}^d.
\]
Furthermore, 
\begin{equation}\label{eq:diffH}
\|\Hb_{n+1} - \Hb_n\| \le C d (\|\p_{n+2}- \p_{n+1}\|+ \|\p_{n+1}-\p_n\| ).
\end{equation}
\end{lemma}
\begin{proof}
Write $\nabla \LL (\p) =[f_1 (\p), \cdots, f_d(\p)]^T$, where $f_i (\p) := \frac{\partial \LL}{\partial w_i}$. Then, applying Taylor's theorem to each $f_i$, we have
\[
f_i (\p_{n+1})- f_i (\p_n) =\nabla f_i (\vy_i)^T (\p_{n+1}-\p_n),
\]
where  $\vy_i=(1-\xi_i)\p_{n} + \xi_i \p_{n+1}$ for some $\xi_i \in [0,1]$, which is $n$ dependent. Since
\[
\Hb_n=\left[\frac{\partial^2 \LL}{\partial w_i \partial w_j} (\vy_i) \right]_{i,j=1}^d
=\left(
  \begin{array}{c}
    \nabla \frac{\partial \LL}{\partial w_1} (\vy_1)^T \\
    \nabla \frac{\partial \LL}{\partial w_2} (\vy_2)^T \\
    \vdots \\
    \nabla \frac{\partial \LL}{\partial w_d}(\vy_d)^T \\
  \end{array}
\right)
=\left(
  \begin{array}{c}
    \nabla f_1 (\vy_1)^T \\
    \nabla f_2 (\vy_2)^T \\
    \vdots \\
    \nabla f_d (\vy_d)^T \\
  \end{array}
\right),
\]
we have 
\[
\nabla \LL (\p_{n+1})- \nabla \LL (\p_n)=\left(
                                         \begin{array}{c}
                                           f_1 (\p_{n+1})- f_1 (\p_n) \\
                                           \vdots \\
                                           f_d (\p_{n+1})- f_d (\p_n) \\
                                         \end{array}
                                       \right)
 = \Hb_n (\p_{n+1}-\p_n).
\]
We have $\Hb_{n+1}$  defined similarly and we can write 
\[
\Hb_{n+1}=\left[\frac{\partial^2 \LL}{\partial w_i \partial w_j} (\vz_i) \right]_{i,j=1}^d
=\left[ \nabla f_1 (\vz_1), \nabla f_2 (\vz_2), \cdots, \nabla f_d (\vz_d)\right]^T,
\]
for some  $\vz_i=(1-\eta_i)\p_{n+1} + \eta_i \p_{n+2}$ with $\eta_i \in [0,1]$.
Now, by the Lipschitz continuity,
\begin{eqnarray*}
\|\nabla f_i (\vz_i)-\nabla f_i (\vy_i)\|_\infty
&\le& C \|\vz_i - \vy_i\| \\
   &=& C\|(1-\eta_i)\p_{n+1} + \eta_i \p_{n+2} - (1-\xi_i)\p_{n} - \xi_i \p_{n+1} \| \\
   &=& C\|\eta_i(\p_{n+2}-\p_{n+1}) + (1-\xi_i)(\p_{n+1}-\p_n)\| \\
   &\le& C(\eta_i\|\p_{n+2}-\p_{n+1}\|+ (1-\xi_i)\|\p_{n+1}-\p_n\| )\\
   &\le& C(\|\p_{n+2}-\p_{n+1}\|+ \|\p_{n+1}-\p_n\| ).
\end{eqnarray*}
Thus $\|\Hb_{n+1} -\Hb_n\| \le \|\Hb_{n+1} -\Hb_n\|_F \le d \max_i \|\nabla f_i (\vz_i)-\nabla f_i (\vy_i)\|_\infty \le  C d (\|\p_{n+2}-\p_{n+1}\|+ \|\p_{n+1}-\p_n\|)$, where $\|\cdot \|_F$ is the Frobenius norm of matrices.
\end{proof}

We now prove Theorem~\ref{thm1}.

\begin{theorem}[Theorem~\ref{thm1} Restate]
Consider the adaptive momentum method FRGD \eqref{eq:FRGD} for $\LL (\vw)$. Assume that $\Hb_{ij}(\vw):=\frac{\partial^2 \LL}{\partial w_i \partial w_j}$ is Lipschitz continuous with the Lipschiz constant $C$, i.e. $\|\Hb_{ij}(\vw)- \Hb_{ij}(\Tilde{\vw})\| \le C \|\vw-\Tilde{\vw}\|$ for all $1\le i, j \le d$. Let $\Hb_n$ be as defined in \eqref{eq:Hn}.
Assume for some $K>0$ that the eigenvalues of $\Hb_n$ are on $[\lambda_{\min}, \lambda_{\max}]$  with $\lambda_{\min}>0$ for $n \le K$  and $\alpha \le \frac{ \lambda_{\min}}{ (\lambda_{\max}^2+2 C d \|\rr_0\|) K^2}$. Then
\[
\rr_n^T \pp_n>0 \;\;\mbox{ and }\;\;
\|\rr_{n}\| \le \sqrt{1-\alpha \lambda_{\min}}  \|\rr_{n-1}\|,
\]
where $\rr_n = \nabla \LL (\p_n)$ and $k\le K$.
\end{theorem}
\begin{proof}
Consider $\p_{n-1}, \p_n, \p_{n+1}$. By Lemma \ref{lem1}, there is $\Hb_n$ such that
\begin{equation}\label{eq:rk}
\rr_{n+1} - \rr_n = \nabla \LL (\p_{n+1}) -\nabla \LL (\p_n) = \Hb_n (\p_{n+1} -\p_n) =-\alpha \Hb_n \pp_n.
\end{equation}
Furthermore
\begin{equation}\label{eq:Hpert}
\|\Hb_{n} - \Hb_{n-1} \| \le C d  (\| \p_{n+1} -\p_{n}\|+  \| \p_{n} -\p_{n-1}\|)
= C d \alpha (\| \pp_{n}\|+  \|\pp_{n-1}\|).
\end{equation}

Now we prove by induction in $n$ that, for $1\le n \le K$,
\begin{equation}\label{eq:main}
\frac{\rr_{n-1}^T\Hb_{n-1}\pp_{n-1}}{\rr_{n-1}^T \rr_{n-1}} \ge \lambda_{\min}, \;
\|\pp_{n-1}\| \le n\|\rr_{n-1}\|, \;  \mbox{ and } \; \|\rr_{n}\| \le \sqrt{1-\alpha \lambda_{\min}} \|\rr_{n-1}\|.
\end{equation}

First, since $\pp_0= \rr_0$ and $\rr_1 =\rr_0 - \alpha \Hb_0 \pp_0= \rr_0- \alpha \Hb_0 \rr_0$, we have
$\rr_0^T\Hb_0 \pp_0 =\rr_0^T\Hb_0 \rr_0 \ge \lambda_{\min}\rr_0^T \rr_0 $, $\|\pp_{0}\| = \|\rr_{0}\|$, and
\begin{eqnarray*}
\rr_1^T \rr_1 &=& \rr_0^T \rr_0 -2 \alpha \rr_0^T \Hb_0 \rr_0 + \alpha^2 \rr_0^T \Hb_0^2 \rr_0\\
    &\le &  \rr_0^T \rr_0 -2 \alpha \lambda_{\min} \rr_0^T  \rr_0 + \alpha^2 \lambda_{\max}^2 \rr_0^T \rr_0 \\
    &\le &  (1- \alpha \lambda_{\min}) \rr_0^T \rr_0,
\end{eqnarray*}
where we have used $- \lambda_{\min}+ \alpha \lambda_{\max}^2 \le 0$.
So, \eqref{eq:main} holds for $n=1$.

Assume that \eqref{eq:main} holds for some $n \le K-1$. We prove it for $n+1$. Using \eqref{eq:FRGD} and the induction assumption, we have  $\beta_n = \frac{ \|\rr_n\|^2 }{\|\rr_{n-1}\|^2} \le 1$ and then 
\begin{eqnarray*}
\pp_n^T \pp_n &=& \rr_n^T \rr_n + 2 \beta_n \rr_{n}^T \pp_{n-1} + \beta_n^2 \pp_{n-1}^T  \pp_{n-1} \\
    &\le & \|\rr_n\|^2 + 2 \frac{ \|\rr_n\|^2 }{\|\rr_{n-1}\|^2} \|\rr_{n}\| \| \pp_{n-1}\| + \frac{ \|\rr_n\|^2 }{\|\rr_{n-1}\|^2} \| \pp_{n-1}\|^2 \\
    &\le & \|\rr_n\|^2 + 2  \|\rr_n\|^2 n  + \|\rr_n\|^2 n^2  \\
    &= & (1+ n)^2 \|\rr_n\|^2.
\end{eqnarray*}
where we have used $\|\pp_{n-1}\| \le n \|\rr_{n-1}\|$.

Next, using \eqref{eq:FRGD} and \eqref{eq:rk}, we have
\begin{eqnarray}
\rr_n^T \Hb_{n}\pp_n &=& \rr_n^T \Hb_{n}\rr_n + \beta_n \rr_n^T \Hb_{n}\pp_{n-1} \nonumber \\
    &= &  \rr_n^T \Hb_{n}\rr_n + \beta_n \rr_{n-1}^T \Hb_{n}\pp_{n-1} - \alpha \beta_n \pp_{n-1}^T \Hb_{n-1} \Hb_{n} \pp_{n-1} \nonumber \\
    &\ge & \lambda_{\min} \rr_n^T \rr_n + \rr_n^T \rr_n \frac{ \rr_{n-1}^T \Hb_{n-1}\pp_{n-1}}{\rr_{n-1}^T \rr_{n-1}}+ \rr_n^T \rr_n \frac{ \rr_{n-1}^T (\Hb_n-\Hb_{n-1})\pp_{n-1}}{\rr_{n-1}^T \rr_{n-1}} \nonumber \\
    & & - \alpha \frac{ \rr_n^T \rr_n }{\rr_{n-1}^T \rr_{n-1}} \lambda_{\max}^2  \pp_{n-1}^T  \pp_{n-1} \nonumber \\
    &\ge & \lambda_{\min} \rr_n^T \rr_n + \rr_n^T \rr_n \lambda_{\min} -\rr_n^T \rr_n \frac{ \|\pp_{n-1}\|}{\|\rr_{n-1}\|} \|\Hb_n-\Hb_{n-1}\|
     - \alpha \lambda_{\max}^2 \rr_n^T \rr_n \frac{ \|\pp_{n-1}\|^2}{\|\rr_{n-1}\|^2} \nonumber \\
    &\ge & \lambda_{\min} \rr_n^T \rr_n + \rr_n^T \rr_n \lambda_{\min} - \rr_n^T \rr_n n Cd \alpha (\| \pp_{n}\|+  \| \pp_{n-1}\|)
     - \alpha \lambda_{\max}^2 \rr_n^T \rr_n n^2 \nonumber \\
    &\ge & \lambda_{\min} \rr_n^T \rr_n + \rr_n^T \rr_n \lambda_{\min} -\alpha \rr_n^T \rr_n  C d n ((n+1)\|\rr_{n}\|+ n \| \rr_{n-1}\|)
     - \alpha \lambda_{\max}^2 \rr_n^T \rr_n n^2 \nonumber \\
    &\ge & \lambda_{\min} \rr_n^T \rr_n + \rr_n^T \rr_n \lambda_{\min} -\alpha \rr_n^T \rr_n C d 2K^2 \|\rr_{0}\|
     - \alpha \lambda_{\max}^2 \rr_n^T \rr_n K^2 \nonumber \\
    &\ge & \lambda_{\min} \rr_n^T \rr_n, \label{eq:rap}
\end{eqnarray}
where the last inequality follows from the condition on $\alpha$.

Finally, using the two inequalities above, we have
\begin{eqnarray*}
\rr_{n+1}^T \rr_{n+1} &=& \rr_n^T \rr_n -2 \alpha \rr_n^T \Hb_{n}\pp_{n} + \alpha^2  \pp_{n}^T \Hb_{n}^2 \pp_{n} \\
    &\le & \rr_n^T \rr_n -2 \alpha \lambda_{\min} \rr_n^T \rr_n  + \alpha^2 \lambda_{\max}^2 \pp_{n}^T \pp_{n} \\
    &\le & \rr_n^T \rr_n -2 \alpha \lambda_{\min} \rr_n^T \rr_n  + \alpha^2 \lambda_{\max}^2 (n+1)^2 \rr_n^T \rr_n \\
    &\le & (1-\alpha \lambda_{\min})  \rr_n^T \rr_n - \alpha (\lambda_{\min} - \alpha \lambda_{\max}^2 (n+1)^2) \rr_n^T \rr_n\\
    &\le & (1-\alpha \lambda_{\min})  \rr_n^T \rr_n,
\end{eqnarray*}
where we note that  $n+1 \le K$ and hence $\lambda_{\min} - \alpha \lambda_{\max}^2 (n+1)^2 \ge 0$.
This completes the proof of \eqref{eq:main}.

We now prove $\rr_n^T \pp_n \ge \rr_n^T \rr_n >0$ by induction. The case $n=0$ is trivial and assume it hold for $n-1$. Then 
\begin{eqnarray*}
\rr_n^T \pp_n &=& \rr_n^T \rr_n + \beta_n \rr_n^T \pp_{n-1} \nonumber \\
    &= &  \rr_n^T \rr_n + \beta_n \rr_{n-1}^T \pp_{n-1} - \alpha \beta_n \pp_{n-1}^T H_{n-1}  \pp_{n-1} \nonumber \\
    &= & \rr_n^T \rr_n + \rr_n^T \rr_n \frac{ \rr_{n-1}^T \pp_{n-1}}{\rr_{n-1}^T \rr_{n-1}}- \alpha  \rr_n^T \rr_n \frac{ \pp_{n-1}^T H_{n-1}  \pp_{n-1}}{\rr_{n-1}^T \rr_{n-1}} \nonumber \\
    &\ge & \rr_n^T \rr_n + \rr_n^T \rr_n 
     - \alpha \lambda_{\max} \rr_n^T \rr_n \frac{ \|\pp_{n-1}\|^2}{\|\rr_{n-1}\|^2} \nonumber \\
    &\ge &  2 \rr_n^T \rr_n  -\alpha \lambda_{\max}  n^2  \rr_n^T \rr_n \nonumber \\
    &\ge & \rr_n^T \rr_n  
\end{eqnarray*}
where the last inequality follows from $\alpha \le \frac{ \lambda_{\min}}{ (\lambda_{\max}^2+2 C d \|\rr_0\|) K^2}
< \frac{ \lambda_{\min}}{ \lambda_{\max}^2 K^2} \le \frac{1}{ \lambda_{\max} K^2}$. This completes the proof the theorem. 
\end{proof}

Next, we present the proof of Theorem \ref{thm2}.

\begin{theorem}[Theorem~\ref{thm2} Restate]
Consider the CG-momentum method \eqref{eq:FRGD} for $\LL (\vw)= \frac{1}{2}\vw^T \Ab\vw - \vb^T\vw$.
Let $\Zb_n = [\vz_0, \vz_1, \cdots, \vz_{n-1}]$ where $\vz_i=\rr_i/\|\rr_i\|$. 
If $\vz_0, \vz_1, \cdots, \vz_n$ are linearly independent, then
\begin{equation}
\|  \rr_{n}\| \le
2(1+K_n) \left( \frac{\sqrt{\kappa}-1}{\sqrt{\kappa}+1}\right)^n \| \rr_0\|,
\end{equation}
where $K_n \le n(1+ \frac{n}{2} \rho) \|A\| \kappa (\Zb_{n+1})$, $\rho=\max_{0\le j <i \le n-1} (\|\rr_{i}\|^2)/(\|\rr_{j}\|^2)$, and $\kappa$ is the spectral condition number of $A$ and $\kappa (\Zb_{n+1})$ is the spectral condition number of $\Zb_{n+1}$.
\end{theorem}
\begin{proof}
Let $\Rb_n = [\rr_0, \rr_1, \cdots, \rr_{n-1}]$, $\Db_n = \mbox{diag} \{\|\rr_0\|, \|\rr_1\|, \cdots, \|\rr_{n-1}\| \}$ and $\Pb_n = [\pp_0, \cdots, \pp_{n-1}]$. Then $\Zb_n= \Rb_n \Db_n^{-1}$.  Using $\rr_{k+1}= \rr_k - \alpha \Ab \pp_{k}$, we have
\begin{equation*}
\alpha \Ab\Pb_n =[\rr_0-\rr_1, \rr_1-\rr_2, \cdots, \rr_{n-1}-\rr_{n}] = \Rb_{n}\Lb_n - \rr_{n} e_n^T =\Zb_n \Db_n\Lb_n - \rr_{n} \ve_n^T,
\end{equation*}
and using $\pp_{k}= \rr_k + \beta_{k} \pp_{k-1}$
\begin{equation*}
\Zb_n =\Rb_n\Db_n^{-1}  = \Pb_n \Ub_n\Db_n^{-1},
\end{equation*}
where $\ve_n=[0, \cdots, 0, 1]^T$ and $\Lb_n$ is the $n\times n$ lower bidiagonal matrix with $1$ on the diagonal and $-1$ on the subdiagonal, and $\Ub_n$ is the upper bidiagonal matrix with 1 on the diagonal and $-\beta_1, \cdots, -\beta_{n-1}$ on the superdiagonal. Combining the two equations, we obtain
\[
\Ab\Zb_n = \Zb_{n} {\Tb}_n -  \frac{1}{\alpha'} \frac{\rr_{n}}{\|\rr_0\|} \ve_n^T, 
\]
where $\Tb_n=\frac{1}{\alpha} \Db_n\Lb_n \Ub_n\Db_n^{-1} $ and $\alpha'=\alpha \|\rr_{n-1}\|/\|\rr_0\|.$ Note that $\alpha'= \ve_n^T \Tb_n^{-1} \ve_1$. Apply Theorem 3.5 of \cite{Tong2000AnalysisOT} (with $\hat{\Delta}_n$ there equal to 0 and the indexes shifted by 1) to the above equation, we have
\begin{equation}\label{eq:polybound}
\|  \rr_{n}\| \le
(1+K_n) \min_{p\in {\cal P}_n, p(0)=1} \|p(\Ab) \rr_0\|,
\end{equation}
where $K_n = \| \Ab \Zb_n \Tb_n^{-1}[\Ib_n \;\; 0] \Zb_{n+1}^+\| \le  \| \Ab \| \| \Tb_n^{-1}\| \| \Zb_n\| \|  \Zb_{n+1}^+\|$ and ${\cal P}_n$ is the set of polynomials of degree $n$.

Note that $\beta_k =\frac{\|\rr_{k}\|^2}{\|\rr_{k-1}\|^2}$. Write $\Tb_n=\frac{1}{\alpha} \Db_n\Lb_n\Db_n^{-1} \Db_n \Ub_n\Db_n^{-1} =\frac{1}{\alpha} \hat{\Lb}_n \hat{\Ub}_n$, where
\begin{eqnarray*}
\hat{\Lb}_n:= \Db_n\Lb_n\Db_n^{-1} &=& \left( \begin{array}{ccccc}
     1 & { }  &  { }  &{ }    & { }             \\
     -\frac{\|\rr_1\|}{\|\rr_0\|} & 1 & { } & { }   & { }           \\
      { }     & -\frac{\|\rr_2\|}{\|\rr_1\|}    & \ddots & { } & { }      \\
      & { }     & \ddots & 1 & { } \\
      & { }       & { }   & -\frac{\|\rr_{n-1}\|}{\|\rr_{n-2}\|}     & 1
     \end{array} \right) \\
     &=& \left( \begin{array}{cccccc}
     1 & { }  &  { }  &{ }    & { }            \\
     -\sqrt{\beta_1} & 1 & { } & { }   & { }         \\
      { }     & -\sqrt{\beta_2}    & \ddots & { } & { }      \\
      { }     & { }       & \ddots & 1 & { } \\
      { }     & { }       & { }   & -\sqrt{\beta_{n-1}}     & 1
     \end{array} \right),
\end{eqnarray*}
and
\begin{eqnarray*}
\hat{\Ub}_n:=\Db_n \Ub_n\Db_n^{-1} &=&\left( \begin{array}{cccccc}
     1 & -\beta_1\frac{\|\rr_0\|}{\|\rr_1\|} &  { }  &{ }    & { }        \\
     { }& 1   & -\beta_2\frac{\|\rr_1\|}{\|\rr_2\|} & { }   & { }       \\
      { }     & { }   & \ddots & \ddots & { }           \\
      { }     & { }     & { } & 1 & -\beta_{n-1}\frac{\|\rr_{n-2}\|}{\|\rr_{n-1}\|} \\
      { }     & { }     & { }   & { }     & 1
     \end{array} \right)  \\
   &=& \left( \begin{array}{cccccc}
     1 & -\sqrt{\beta_1}  &  { }  &{ }    & { }          \\
     { }& 1   & -\sqrt{\beta_2} & { }   & { }         \\
      { }     & { }   & \ddots & \ddots & { }        \\
      { }     & { }       & { } & 1 & -\sqrt{\beta_{n-1}} \\
      { }     & { }   & { }   & { }     & 1
     \end{array} \right) =\hat{\Lb}_n^T.
\end{eqnarray*}
Then $\hat{\Lb}_n^{-1}$ is a lower triangular matrix with the diagonals being 1 and with the $(i, j)$ entry being $\sqrt{\beta_{j} \beta_{j+1} \cdots \beta_{i-1}} = \|\rr_{i-1}\|/\|\rr_{j-1}\|$ for $i > j$. Then bounding $\|\rr_{i-1}\|^2/\|\rr_{j-1}\|^2$ by $\rho$, we have $\|\hat{\Lb}_n^{-1}\|_F^2 \le n+n(n-1) \rho /2$. So,
$\|\Tb_n^{-1}\| =\alpha \|\hat{\Lb}_n^{-1}\|^2 \le \alpha \|\hat{\Lb}_n^{-1}\|_F^2 \le \alpha n(1+ n \rho /2)$.
Combining this with $\|\Zb_n\| \|  \Zb_{n+1}^+\|\le \|  \Zb_{n+1}\|\|\Zb_{n+1}\| = \kappa (\Zb_{n+1}) $ results in $K_n \le n\alpha (1+ \frac{n}{2} \rho) \|\Ab\| \kappa (\Zb_{n+1})$. Finally, the bound follows from the standard CG convergence bound \cite[p.215]{doi:10.1137/1.9780898718003}  that shows
\[
\min_{p\in {\cal P}_n, p(0)=1} \|p(\Ab) \rr_0\| \le \min_{p\in {\cal P}_n, p(0)=1} \max_i |p(\lambda_i)|\, \|\rr_0\| \le 2 \left( \frac{\sqrt{\kappa}-1}{\sqrt{\kappa}+1}\right)^n \| \rr_0\|,
\]
where $\lambda_i$ (for $1\le i \le d$) are eigenvalues of $\Ab$.
\end{proof}


\section{Adversarial Attacks}\label{sec:appendix:attacks}
We focus on the $\ell_\infty$ norm-based FGSM, IFGSM, and C\&W white-box attacks.
For a given image-label pair $\{\vx, y\}$, a given ML model $g(\vx, \vw)$, and the associated loss $\LL(\vx, y):=\LL(g(\vx, \vw), y)$:
\begin{itemize}[leftmargin=*]
\item Fast gradient sign method (FGSM) searches an adversarial, $\vx'$, within an $\ell_\infty$-ball as
\begin{equation*}
\label{Eq:FGSM}
{\footnotesize \vx'=\vx + \epsilon\cdot {\rm sign}\left(\nabla_\vx \LL(\vx, y)\right),}
\end{equation*}
and we set $\epsilon=8/255$ in all of our experiments.

\item 
Iterative FGSM (IFGSM$^{M}$) \cite{Goodfellow:2014AdversarialTraining} iterates FGSM and clip the range as
\begin{equation*}
\label{Eq:IFGSM}
{\footnotesize \vx^{(m)} = {\rm Clip}_{\vx, \epsilon}\left\{ \vx^{(m-1)} + \alpha\cdot {\rm sign} \left(\nabla_{\vx^{(m-1)}} \LL(\vx^{(m-1)}, y)\right) \right\},\ \mbox{w/}\ \vx^{(0)} = \vx, \ m=1, \cdots, M,}
\end{equation*}
and we set $\epsilon = 8/255$ and $\alpha=1/255$ in IFGSM attacks with different number of iterations.

\item C\&W attack \cite{CWAttack:2016} searches the minimal perturbation ($\delta$) attack as
\begin{equation*}
\label{Eq:CW}
{\footnotesize \min_{\boldsymbol{\delta}} ||\boldsymbol{\delta}||_\infty,\ \ \mbox{subject to}\ \ g(\vw, \vx+\boldsymbol{\delta}) = t, \; \vx+\boldsymbol{\delta} \in [0, 1]^d,\ \mbox{for}\ \forall t\neq y,}
\end{equation*}
\end{itemize}
we use the same setting as that used in \cite{Wang:2018EnResNet} for C\&W attack.

\end{document}